%% file: densepeds_main.tex
\documentclass[letterpaper, 10 pt, conference]{ieeeconf}
\IEEEoverridecommandlockouts                              % This command is only needed if 
\pdfoutput=1
\usepackage{color}
\usepackage{times}  %Required
\usepackage{url}  %Required
\usepackage{lipsum}
\usepackage{array}

\usepackage{amsmath, amsthm, amssymb, mathtools}
\usepackage{graphicx}  %Required
\usepackage{ltablex, multirow}
\usepackage{algorithm}
\usepackage[noend]{algpseudocode}
\usepackage{amsfonts,bbold}
\usepackage{soul}
\newtheorem{definition}{Definition}[section]

\DeclareMathOperator*{\argmin}{arg\,min}

\newtheorem{prop}{Proposition}[section]
\newtheorem{lemma}{Lemma}[section]
\newcolumntype{L}[1]{>{\raggedright\let\newline\\\arraybackslash\hspace{0pt}}m{#1}}
\newcolumntype{C}[1]{>{\centering\let\newline\\\arraybackslash\hspace{0pt}}m{#1}}
\newcolumntype{R}[1]{>{\raggedleft\let\newline\\\arraybackslash\hspace{0pt}}m{#1}}

\usepackage{color}

\newcommand{\xpi}[1]{#1_{p_i}}
\newcommand{\xhj}[1]{#1_{h_j}}
\newcommand{\msk}{f^M_{h_j}}
\newcommand{\fst}{f^F_{h_j}}
\newcommand{\ft}{f^T_{\p}}
\newcommand{\p}{p_i}
\newcommand{\h}{h_j}
\newcommand{\ith}{i^{\textrm{th}}}
\newcommand{\jth}{j^{\textrm{th}}}
\newcommand{\tim}[1]{\textrm{#1}}
\newcommand{\mc}[1]{\mathcal{#1}}
\newcommand{\bb}[1]{\mathbb{#1}}

\begin{document}
%%%%%%%%% TITLE
\title{DensePeds: Pedestrian Tracking in Dense Crowds Using Front-RVO and Sparse Features}

\author{
Rohan Chandra$^1$, Uttaran Bhattacharya$^1$, Aniket Bera$^2$, and Dinesh Manocha$^1$
\\
% \IEEEauthorblockA{
% Wherever\\
$^1$University of Maryland, $^2$University of North Carolina
% }
}
\maketitle

\input{0-Abstract.tex}

%%%%%%%%%%%%%%%%%%%%%%%%%%%%%%%%%%%%%%%%%%%%%%%%%%%%%%%%%%%%%%%%%%%%%%%%%%%%%
                                    %INTRODUCTION %
%%%%%%%%%%%%%%%%%%%%%%%%%%%%%%%%%%%%%%%%%%%%%%%%%%%%%%%%%%%%%%%%%%%%%%%%%%%%%
\input{1-Intro.tex}
%%%%%%%%%%%%%%%%%%%%%%%%%%%%%%%%%%%%%%%%%%%%%%%%%%%%%%%%%%%%%%%%%%%%%%%%%%%%%
                                    %RELATED WORK%
%%%%%%%%%%%%%%%%%%%%%%%%%%%%%%%%%%%%%%%%%%%%%%%%%%%%%%%%%%%%%%%%%%%%%%%%%%%%%
\input{2-Related.tex}
%%%%%%%%%%%%%%%%%%%%%%%%%%%%%%%%%%%%%%%%%%%%%%%%%%%%%%%%%%%%%%%%%%%%%%%%%%%%%
                                    %APPROACH%
%%%%%%%%%%%%%%%%%%%%%%%%%%%%%%%%%%%%%%%%%%%%%%%%%%%%%%%%%%%%%%%%%%%%%%%%%%%%%
\input{3-FRVO.tex}

\input{4-DensePeds.tex}
%%%%%%%%%%%%%%%%%%%%%%%%%%%%%%%%%%%%%%%%%%%%%%%%%%%%%%%%%%%%%%%%%%%%%%%%%%%%%
                                    %NETWORK%
%%%%%%%%%%%%%%%%%%%%%%%%%%%%%%%%%%%%%%%%%%%%%%%%%%%%%%%%%%%%%%%%%%%%%%%%%%%%%
% \input{5-Network.tex}
%%%%%%%%%%%%%%%%%%%%%%%%%%%%%%%%%%%%%%%%%%%%%%%%%%%%%%%%%%%%%%%%%%%%%%%%%%%%%
                                    %EXPERIMENTS %
%%%%%%%%%%%%%%%%%%%%%%%%%%%%%%%%%%%%%%%%%%%%%%%%%%%%%%%%%%%%%%%%%%%%%%%%%%%%%
\input{6-Experiments.tex}
\input{7-conc.tex}
{\small
\bibliographystyle{ieeetr}
\bibliography{refs}
}
\end{document}

%% file: 0-Abstract.tex
\begin{abstract}
We present a pedestrian tracking algorithm, \textit{DensePeds}, that tracks individuals in highly dense crowds (\textgreater 2 pedestrians per square meter). Our approach is designed for videos captured from front-facing or elevated cameras. We present a new motion model called Front-RVO (FRVO) for predicting pedestrian movements in dense situations using collision avoidance constraints and combine it with state-of-the-art Mask R-CNN to compute sparse feature vectors that reduce the loss of pedestrian tracks (false negatives). We evaluate DensePeds on the standard MOT benchmarks as well as a new dense crowd dataset. In practice, our approach is 4.5$\times$ faster than prior tracking algorithms on the MOT benchmark and we are state-of-the-art in dense crowd videos by over 2.6\% on the absolute scale on average.

% We generate ``Segmented Boxes" using Mask R-CNN to produce simple, yet powerful, detection features to significantly improve accuracy. In addition to the standard MOT benchmarks, we also evaluate and present qualitative results on highly dense datasets. Our approach is 18X faster than prior online methods. In addition to scoring the lowest number of false negatives against both online and offline methods, we also improve the number of successfully tracked pedestrians and  reduce the number of lost tracks over prior online methods by more than $5$\% and increase precision by 3\%.
% Compared to offline methods, we reduce the number of lost tracks by 7\% and increase precision by 3\%.
\end{abstract}

%% file: 1-Intro.tex
\section{Introduction}
Pedestrian tracking is the problem of maintaining the consistency in the temporal and spatial identity of a person in an image sequence or a crowd video. This is an important problem that helps us not only extract trajectory information from a crowd scene video but also helps us understand high-level pedestrian behaviors \cite{bera2017aggressive}. Many applications in robotics and computer vision such as action recognition and collision-free navigation and trajectory prediction~\cite{traPHic} require tracking algorithms to work accurately in real time \cite{teichman2011practical}.
% In pedestrian tracking, the goal is to link each person to the same trajectory across an image sequence.
%However, tracking multiple people in dense scenes has proven to be a difficult task due to occlusions and complex inter-pedestrian dynamics. 
Furthermore, it is crucial to develop general-purpose algorithms that can handle front-facing cameras (used for robot navigation or autonomous driving) as well as elevated cameras (used for urban surveillance), especially in densely crowded areas such as airports, railway terminals, or shopping complexes.
% Pedestrian tracking is also actively used in surveillance \cite{benfold2011stable}. 

Closely related to pedestrian tracking is pedestrian detection, which is the problem of detecting multiple individuals in each frame of a video. Pedestrian detection has received a lot of traction and gained significant progress in recent years. However, earlier work in pedestrian tracking
% zhang2013structure,zhang2014preserving,hu2012single,yang2007game
\cite{zhang2013structure,hu2012single} did not include pedestrian detection. These tracking methods require manual, near-optimal initialization of each pedestrian's state information in the first video frame. 
%This involves manually counting and measuring each person's location. 
Further, sans-detection methods need to know the number of pedestrians in each frame apriori, so they do not handle cases in which new pedestrians enter the scene during the video. Tracking by detection overcomes these limitations by employing a detection framework to recognize pedestrians entering at any point of time during the video and automatically initialize their state-space information.
%Detection-based methods automatically count pedestrians, label their locations, and are robust to the number of pedestrians in the scene.
% The main advantage of this paradigm is that it eliminates many of the limitations imposed by methods that do not use pedestrian detection.
% In practice, however, current algorithms based on the tracking-by-detection paradigm face several performance issues. Most prominent among these are a) real-time computation, b) good results without requiring high-quality appearance features, and c) performing well in scenes with high density; that is, more than two pedestrians per square meter. One key issue in dense videos is estimating the motion of pedestrians between successive frames due to the difficulty of navigation for each pedestrian and occlusion issues.

However, tracking pedestrians in dense crowd videos where there are 2 or more pedestrians per square meter remains a challenge for the tracking-by-detection literature. These videos suffer from severe occlusion, mainly due to the pedestrians walking extremely close to each other and frequently crossing each other's paths. This makes it difficult to track each pedestrian across the video frames. Tracking-by-detection algorithms compute a bounding box around each pedestrian. Because of their proximity in dense crowds, the bounding boxes of nearby pedestrians overlap which affects the accuracy of tracking algorithms.

\begin{figure}
    \centering
    \includegraphics[width=\columnwidth]{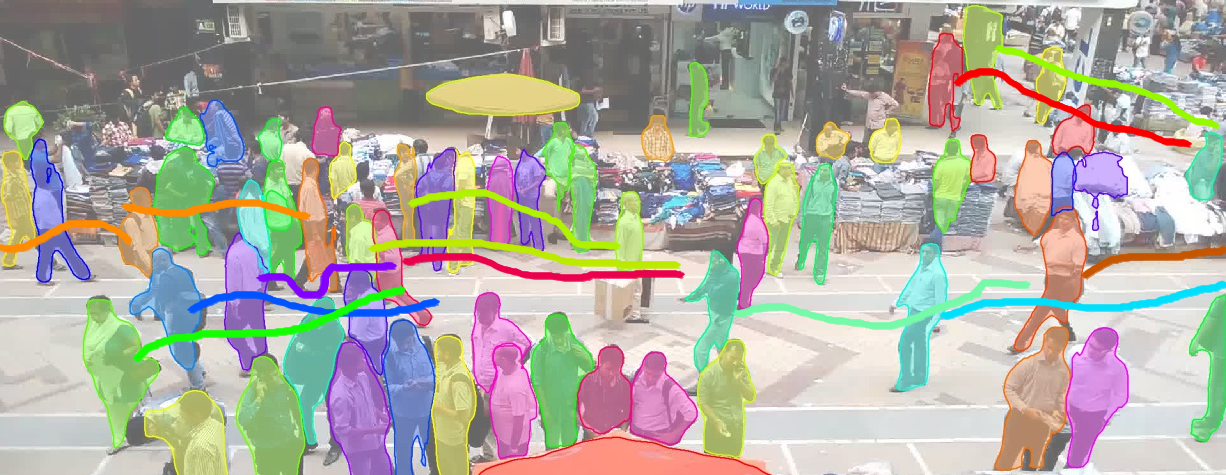}
    % \vspace{-10pt}
    \caption{Performance of our pedestrian tracking algorithm on the NPLACE-1 sequence with over 80 pedestrians per frame. The colored tracks associated with each moving pedestrian are linked to a unique ID. DensePeds achieves an accuracy of up to 85.5\% on our new dense crowd dataset and improves tracking accuracy by 2.6\% over state of the art methods on the average. This is equivalent to an average rank difference of 17 on the MOT benchmark.}
    % Our algorithm is 6\% more accurate in terms of the number of pedestrians successfully tracked and 3\% higher in precision over prior online methods published on the MOT benchmark.}
    \label{cover}
    % \vspace{-20pt}
\end{figure}

\textbf{Main Contributions.} Our main contributions in this work are threefold.
\begin{enumerate}
    \item We present a pedestrian tracking algorithm called DensePeds that can efficiently track pedestrians in crowds with 2 or more pedestrians per square meter. We call such crowds ``dense''. We introduce a new motion model called Frontal Reciprocal Velocity Obstacles (FRVO), which extends traditional RVO \cite{van2011reciprocal} to work with front- or elevated-view cameras. FRVO uses an elliptical approximation for each pedestrian and estimates pedestrian dynamics (position and velocity) in dense crowds by considering intermediate goals and collision avoidance constraints.
    
    \item We combine FRVO with the state-of-the-art Mask R-CNN object detector~\cite{maskrcnn} to compute sparse feature vectors. We show analytically that our sparse feature vector computation reduces the probability of the loss of pedestrian tracks (false negatives) in dense scenarios. We also show by experimentation that using sparse feature vectors makes our method outperform state-of-the-art tracking methods on dense crowd videos, without any additional requirement for training or optimization.
    
    \item We make available a new dataset consisting of dense crowd videos. Available benchmark crowd videos are extremely sparse (less than 0.2 persons per square meter). Our dataset, on the other hand, presents a more challenging and realistic view of crowds in public areas in dense metropolitans.
\end{enumerate}

On our dense crowd videos dataset, our method outperforms the state-of-the-art, by 2.6\% and reduce false negatives by 11\% on our dataset. We also validate the benefits of FRVO and sparse feature vectors through ablation experiments on our dataset. Lastly, we evaluate and compare the performance of DensePeds with state-of-the-art online methods on the MOT benchmark \cite{mot16} for a comprehensive evaluation. The benefits of our method do not hold for sparsely crowded videos such as the ones in the MOT benchmark since FRVO is based on reciprocal velocities of two colliding pedestrians. Unsurprisingly, we do not outperform the state-of-the-art methods on the MOT benchmark since the MOT sequences do not contain many colliding pedestrians, thereby rendering FRVO ineffective. Nevertheless, our method still has the lowest false negatives among all the available methods. Its overall performance is state of the art on dense videos and is in the top 13\% of all published methods on the MOT15 and top 20\% on MOT16.
% We compare our approach with methods that have average ranks (considering all metrics together) \textit{higher} than ours on the following four metrics: number of successfully tracked pedestrians, number of identity switches, number of tracks lost, and precision.
% Compared to prior online methods, we achieve up to 17\% improvement in accuracy over state-of-the-art methods. As compared to prior offline methods, we achieve up to 30\% improvement in accuracy. Finally, among linear model-based methods, we achieve up to 13\% improvement in accuracy. Our method is faster than previous methods by up to 18X and provides real-time performance.

% The rest of the paper is organized as follows: In Section~\ref{Sec2}, we survey the state-of-the-art works in pedestrian tracking and motion models. We give an overview of FRVO in Section~\ref{Sec3} and use it to compute the sparse features in Section~\ref{Sec4}. We present the experimental results in Section~\ref{Sec5} and compare our results with prior methods.

%% file: 2-Related.tex
\section{Related Work}\label{Sec2}
The most relevant prior work for our work is on object detection, pedestrian tracking, and motion models used in pedestrian tracking, which we present below.
% Each of these fields is well-studied, and a comprehensive overview for each of them is well beyond the scope of this paper. We only present a brief overview of prior work on these fields that are most closely related to our work in this paper.
\subsection{Object Detection}
Early methods for object detection include HOG \cite{dalal2005histograms} and SIFT \cite{lowe2004distinctive}, which manually extract features from images. Inspired by AlexNet, \cite{2013arXiv1311.2524G} proposed R-CNN and its variants~\cite{fastrcnn,fasterrcnn} for optimizing the object detection problem by incorporating a selective search. 
% Subsequently, Fast R-CNN \cite{fastrcnn} and Faster R-CNN \cite{fasterrcnn} were proposed to improve the speed of the original R-CNN. 

More recently, the prevalence of CNNs has led to the development of Mask R-CNN \cite{maskrcnn}, which extends Faster R-CNN to include pixel-level segmentation. The use of Mask R-CNN in pedestrian tracking has been limited, although it has been used for other pedestrian-related problems such as pose estimation \cite{li2018pose2seg}.

% A technical report implements pedestrian detection with Mask R-CNN (https://github.com/sahibdhanjal/Mask-RCNN-Pedestrian-Detection), but it only performs detection while ours builds a real-time tracking system.

\subsection{Pedestrian Tracking}
% Offline tracking methods using deep-learning have some overlap with linear motion model-based algorithms, which we previously discussed. We, therefore, only focus on online algorithms in this subsection.
There have been considerable advances in object detection since the advent of to deep learning, which has led to substantial research in the intersection of deep learning and the tracking-by-detection paradigm \cite{rtdl1,rtdl2,rtdl4,rtdl5-online153}. However, \cite{rtdl5-online153,rtdl2} operate at less than one fps on a GPU and have low accuracy on standard benchmarks while \cite{rtdl1} sacrifices accuracy to increase tracking speed. Finally, deep learning methods require expensive computation, often prohibiting real-time performance with inexpensive computing resources. Some recent methods~\cite{deepsort,rt1} achieve high accuracy but may require high-quality, heavily optimized detection features for good performance. For an up-to-date review of tracking-by-detection algorithms, we refer the reader to methods submitted to the MOT \cite{mot16} benchmark.

\subsection{Motion Models in Pedestrian Tracking}
Motion models are commonly used in pedestrian tracking algorithms to improve tracking accuracy \cite{cem-lin2,mht-lin1,edmt-lin3,lfnf-lin4}. \cite{mht-lin1} presents a variation of MHT~\cite{mht} and shows that it is at par with the state-of-the-art from the tracking-by-detection paradigm. \cite{lfnf-lin4} uses a motion model to combine fragmented pedestrian tracks caused by occlusion. These methods are based on linear constant velocity or acceleration models. Such linear models, however, cannot characterize pedestrian dynamics in dense crowds~\cite{bera2014realtime}. RVO \cite{van2011reciprocal} is a non-linear motion model that has been used for pedestrian tracking in dense videos to compute intermediate goal locations, but it only works with top-facing videos and circular pedestrian representations. RVO has been extended to tracking road-agents such as cars, buses, and two-wheelers, in addition to pedestrians, using a linear runtime motion model that considers both collision avoidance and pair-wise heterogeneous interactions between road-agents~\cite{chandra2019roadtrack}. Other motion models used in pedestrian tracking are the Social Force model \cite{bera2014adapt}, LTA \cite{5459260}, and ATTR \cite{yamaguchi2011you}.

There are also many discrete motion models that represent each individual or pedestrian in a crowd as a particle (or as a 2D circle on a plane) to model the interactions. These include models based on repulsive forces~\cite{helbing1995social} and velocity-based optimization algorithms~\cite{karamouzas2009predictive}, ~\cite{van2011reciprocal}. More recent discrete approaches are based on short-term planning using a discrete approach~\cite{antonini2006behavioral} and cognitive models~\cite{chung2010mobile}. However, these methods are based on circular agent representation and do not work well for front-facing pedestrians in dense crowd videos as they are overly conservative in terms of motion prediction.
% , or linear trajectory avoidance (LTA)~\cite{pellegrini2009you}.

% \subsection{FrontRVO prior work}

% \vspace{-5pt}
% rtdl1,rtdl2,rtdl4,rtdl5-online153

%% file: 3-FRVO.tex
\section{FRVO: Local Trajectory Prediction}
\label{Sec3}

%Our approach is designed for tracking pedestrians in dense crowds with 2 or more pedestrians per square meters on average.

%One of our goals is to use an accurate motion model to predict the trajectories of pedestrians between successive frames.

%Prior motion mo We achieve this by addressing two main challenges. First, occlusion in dense crowds makes it difficult to maintain correct ID association across frames over a large time-span, increasing the chances of switching the ID label of the person being occluded. Also, if the tracking confidence falls below a threshold, the track is destroyed and becomes a false negative. Our approach addresses both these challenges by 
When walking in dense crowds, pedestrians keep changing their velocities frequently to avoid collisions with other pedestrians. Pedestrians also exhibit local interactions and other collision-avoidance behaviors such as side-stepping, shoulder-turning, and backpedaling~\cite{best2016real}. %Large variations in their appearances, illuminations, and motions make it difficult for most 
As a result, prior motion models with constant velocity or constant acceleration assumptions do not accurately model crowded scenarios. Further, defining an accurate motion model gets even more challenging in front-view videos due to occlusion and proximity issues. For example, in top-view videos, one uses circular representations to model the shape of the heads of pedestrians. It is, generally, physically impossible for two individual's heads to occlude or occupy the same space. In front facing crowd videos, however, occlusion is a common problem. In this section, we introduce our non-linear motion model, Frontal RVO (FRVO), which is designed to work with crowd videos captured using front- or elevated-view cameras.

% Another consequence of this is that we can now use deep CNNs both for object detection and feature extraction. 

%%%%%%%%%%%%%%%%%%%%%%%%%%%%%%%%%%%%%%%%%%%%%%%% NOTATION %%%%%%%%%%%%%%%%%%%%%%%%%%%%%%%%%%%%%%%%%%%%%%% 
\subsection{Notations}
Table~\ref{tab:notations} lists the notations used in this paper. A few important points to keep in mind are:
\begin{itemize}
    \item While $\h$ denotes a detected pedestrian, $\xhj{f}$ denotes the feature vector extracted from the \textit{segmented box} of $\h$.
    \item $\mc{H}_i = \{ \h \ : \ |\p - \h| \leq \rho, \p \in \mc{P}\}$
\end{itemize}
\begin{table}[t]
    \centering
    \begin{tabular}{|R{1.3cm}|L{6.4cm}|}
        \hline
        $\p$ & $\ith$ pedestrian \\
        \hline
        $\h$ & $\jth$ detected pedestrian \\
        \hline
        $\xpi{f}$ \& $\xhj{f}$ & feature vectors corresponding to $\p$ and $\h$ respectively \\
        \hline
        $T$ & total number of frames in the video \\
        \hline
        $\mc{H}$ \& $\mc{P}$ & sets of all pedestrian detections and total pedestrians in a frame respectively \\
        \hline
        $\rho$ & predefined radius around every $\p$ \\
        \hline
        $\mc{H}_i$ & set of all detected persons within a circle of radius $r$ centered around $\p$ \\
        \hline
        $N_i$ & cardinality of the set $\mc{H}_i$ \\
        \hline
        $\mathnormal{d}(a,b)$ & cosine distance metric between the vectors $a$ and $b$, defined as:  $\mathnormal{d}(a,b) = 1-\frac{a^\top b}{\lVert a \rVert\lVert b \rVert}$ \\
        \hline
    \end{tabular}
    \caption{Notations we use in this paper.}
    \label{tab:notations}
\end{table}

%   denote the  and  denote the . Further,  denote the . It is important to note that   image sequence. .  and . . Thus, . Finally, the Cosine metric \cite{cosine} measures the , $a$ and $b$. It is mathematically defined as $\mathnormal{d}(a,b) = 1 - a^Tb$.

% \section{FRVO and Sparse Feature Extraction}

\subsection{Analytical Comparison with RVO} \label{Sec3.2}
%In crowded scenes, the pairwise interactions between pedestrians increase significantly, adding to the complexity of predictive tracking schemes.
\begin{figure}
    \centering
    \includegraphics[width=\columnwidth]{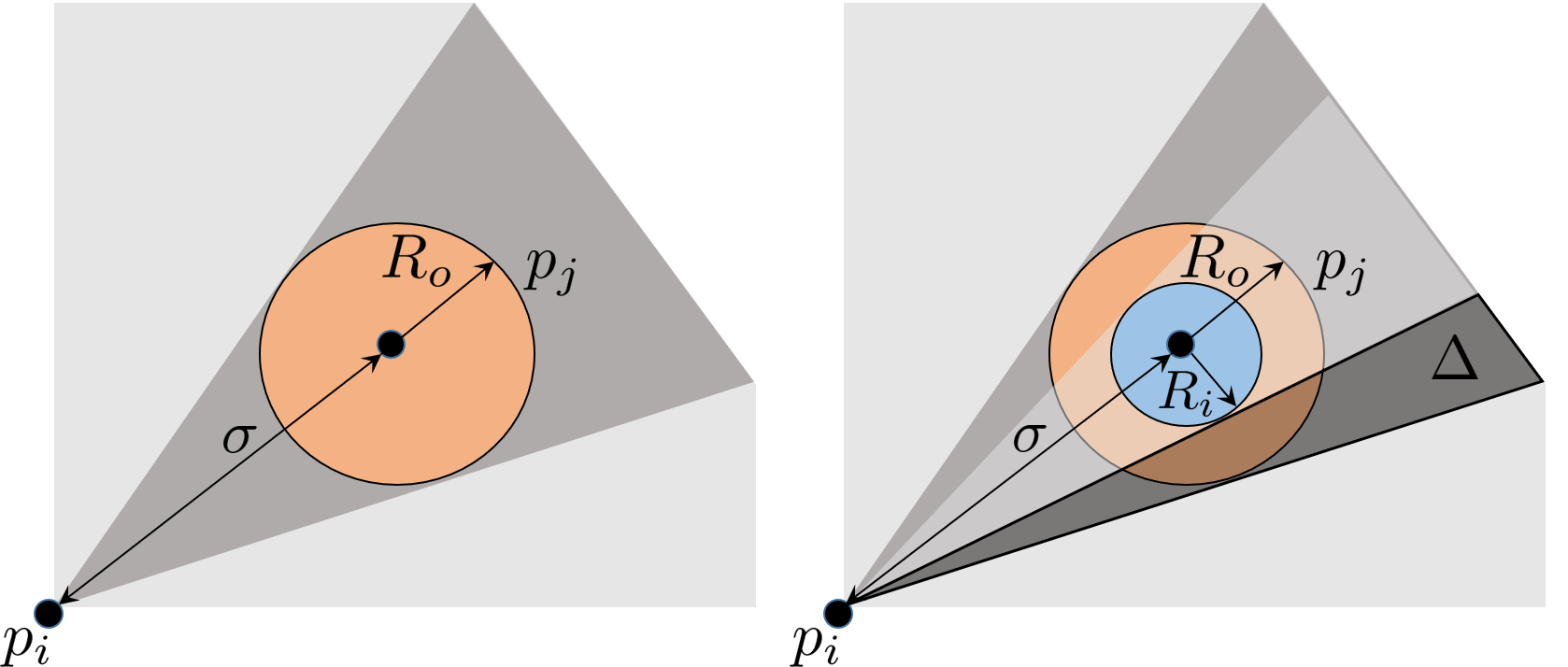}
% \vspace{-15pt}
    \caption{(Left) Standard VO configuration fundamental to RVO, using circular representations for top-view. (Right) RVO for front and elevated videos. The cyan circle and it's corresponding transparent gray cone represent the VO configuration for $\delta$-overlapping circles. RVO considers the highlighted area of the dark gray cone to be part of the collision area whereas it is a collision-free area since it is outside the transparent cone (VO of $\delta$-overlapping circles).}
    % Our algorithm is 6\% more accurate in terms of the number of pedestrians successfully tracked and 3\% higher in precision over prior online methods published on the MOT benchmark.}
    \label{fig:frvo}
\end{figure}

%Prior motion models are mostly defined for circular pedestrians in 2D or spherical pedestrians in 3D. These models are too conservative to be used in real-world videos with front-facing cameras.
% An extension of the RVO model has been proposed for elliptical pedestrians in a plane \cite{van2011reciprocal, best2016real}. 
 
% \subsubsection{Approach}

Our motion model is an extension of the RVO (Reciprocal Velocity Obstacle)~\cite{van2011reciprocal} approach, which can accurately model the trajectories of pedestrians using collision avoidance constraints. However, RVO models pedestrians using circular shapes which result in false overlaps in front- and elevated-view camera videos (Figure~\ref{fig:frvocomp}). A false overlap is essentially a false positive wherein RVO would signal a collision while the actual scene would not contain a collision. We show here that false overlaps cause the accuracy of RVO to drop by an error margin of $\Delta$.

\begin{definition}{\textbf{$\mathbf{\delta}$-overlap:}}
Two circles each of radius $r$ are said to $\delta$-overlap if the length of the line joining their centers is equal to $2r-\delta$, where $0<\delta<2r$.
\label{def}
\end{definition}

Now, consider the standard Velocity Obstacle formulation (the basis of RVO)~\cite{fiorini1998motionVO} in Figure~\ref{fig:frvo}~(left). From the VO formulation, we have $R_o = 2r$ and from \ref{def}, $R_i = 2r - \delta$. Let $\sigma$ denote the distance of $p_j$ from $\p$. Then from simple geometry,

\begin{equation*}
\begin{split}
0 \leq \Delta & \leq \sin^{-1} \textstyle \Big( \frac{2r}{\lVert\sigma\rVert}  \Big) -  \sin^{-1}  \textstyle \Big( \frac{2r-\delta}{\lVert\sigma\rVert}  \Big) \\
%  & = \arcsin \Big( \frac{2r}{\lVert\sigma\rVert}  \Big) -  \arcsin \Big( \dfrac{2r-\delta}{\lVert\sigma\rVert}  \Big) \\
 & \leq \sin^{-1} \textstyle \Big( \frac{2r}{\lVert \sigma \rVert}\sqrt{1-\Big(\frac{\delta-2r}{\lVert \sigma \rVert}\Big)^2} + \frac{\delta-2r}{\lVert \sigma \rVert}\sqrt{1-\Big(\frac{2r}{\lVert \sigma \rVert}\Big)^2}  \Big) \\
\end{split}
\end{equation*}

\noindent $\Delta$ is upper-bounded by a function of $\delta$. Observe that the error bound increases for higher $\delta$. This is interpreted as follows: the more we increase the $\delta$-overlap, the transparent gray cone will correspondingly shrink, and the distance between the RVO generated velocity, and the ground truth velocity will increase.

\begin{figure}[t]
    \centering
    \includegraphics[width=\columnwidth]{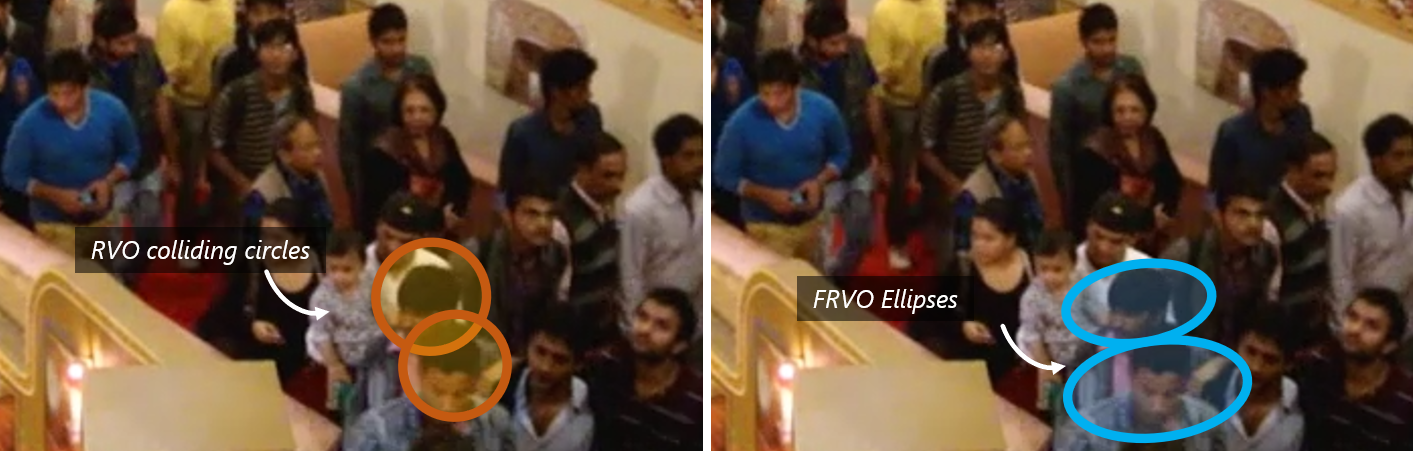}
% \vspace{-15pt}
    \caption{(\textit{left}) A circular representation results in many false $\delta$-overlaps.(\textit{right}) FRVO efficiently models pedestrians using elliptical representations that on average cause $\delta \rightarrow 0$. \textit{Sequence:} IITF-1.}
    \label{fig:frvocomp}
    \end{figure}

% Given each pedestrian's state at a certain time-step, FRVO computes a collision-free state for the next time-step by updating the velocity appropriately. Each pedestrian is represented as an ellipse on a plane, and the resulting velocity obstacle formulation for each pedestrian is computed based on the following parameters:  representative ellipse, maximum speed, neighbor distance, the angle between the ellipse minor axis and the image plane axis, and the time horizon. The time horizon typically corresponds to the time between two successive frames and only future collisions within this horizon are considered for local interactions. \utt{isn't this para redundant? There's a whole subsubsection dedicated to this.}
We therefore require geometric representations for which $\delta \rightarrow 0$. We propose to model each pedestrian using an elliptical pedestrian representation, with the ellipse axes capturing the height of the pedestrian's visible face and the shoulder length. We do not manually differentiate between the major and minor axes since ellipses can be tall or fat depending on the pedestrian orientation, and the particular axis to map either face length or shoulder width will vary. We observe that doing so produces an effective approximation. We only use the face and the shoulder in the FRVO formulation, because occlusion makes it difficult to observe other parts of the body in dense videos reliably. 

\subsection{Computing Predicted Velocities}

Each pedestrian $\p$ is represented using the following 8-dimensional state vector:

\[\Psi_{t} = \big[ x , v , v_{\tim{pref}} , {l} , {w} \big ],\]

\noindent where $x$, $v$ , and $v_{\tim{pref}}$ denote the current position of the pedestrian's center of mass, the current velocity, and the preferred velocity, respectively. $l$ is the height of the pedestrian's visible face, and $w$ captures the shoulder length. $v_{\tim{pref}}$ is the velocity the pedestrian would have taken in the absence of obstacles or colliding pedestrians, computed using the standard RVO formulation.

We assume the pedestrians are oriented towards their direction of motion. For each frame and each pedestrian, we construct the half-plane constraints for each of its neighboring pedestrians and obstacles to predict its motion. We use velocity obstacles to compute the set of permitted velocities for a pedestrian. Given two pedestrians $\p$ and $p_j$, the velocity obstacle of $\p$ induced by $p_j$, denoted by $VO_{\p|p_j}^{\tau}$ , constitutes the set of velocities for $\p$ that would result in a collision with $p_j$ at some time before $\tau$. By definition, pedestrians $\p$ and $p_j$ are guaranteed to be collision-free for at least time $\tau$, if ${{\vec{v}}_{\p}}-{{\vec{v}}_{p_j}}\notin VO_{\p|p_j}^{\tau}$. An pedestrian $\p$ computes the velocity obstacle $VO_{\p|p_j}^{\tau}$  for each of its neighboring pedestrians, $p_j$. The set of permitted velocities for a pedestrian $\p$ is simply the convex region given by the intersection of the half-planes of the permitted velocities induced by all the neighboring pedestrians and obstacles. We denote this convex region for pedestrian $\p$ and time horizon $\tau$ as ${FRVO}_{\p}^{\tau}$. Thus,

\[{FRVO}_{\p}^\tau = \bigcup \limits_{p_j \in \mc{H}_i} {VO}_{\p|p_j}^{\tau} \]
    
    % \Rightarrow {FRVO}_{\p}^\tau &= {FRVO}_{\p}^\tau\bigcup\limits_{O} {VO}_{\p|O}^{\tau}

For each pedestrian $\p$, we compute the new velocity $v_{\text{new}}$ from ${FRVO}_{\p}^{\tau}$  that minimizes the deviation from its preferred velocity $v_{\text{pref}}$.
\begin{equation}
v_{\textrm{new}} = \underset{v}{\arg\max}\left\lVert v - v_{\textrm{pref}} \right\rVert
\end{equation}
such that $v \notin FRVO^\tau_{\p}$.
%%%%%%%%%%%%%%%%%%%%%%%%%%%%%%%%%%%%%%%%%%%%%%%SIMULATION%%%%%%%%%%%%%%%%%%%%%%%%%%%%%%%%%%%%%%%%%%%%%%%%%%%
\input{sim.tex}
%%%%%%%%%%%%%%%%%%%%%%%%%%%%%%%%%%%%%%%%%%%%%%%SIMULATION%%%%%%%%%%%%%%%%%%%%%%%%%%%%%%%%%%%%%%%%%%%%%%%%%%%
In order to compute collision-free velocities, we need to compute the Minkowski Sums of ellipses. In practice, computing the exact Minkowski sums of ellipses is much more expensive as compared to those of circles.
To overcome the complexity of exact Minkowski Sum computation, we compute conservative linear approximations of ellipses \cite{best2016real} and represent them as convex polygons. As a result, the collision avoidance problem reduces to linear programming. 
% We also use a precomputed table of Minkowski Sums based on different orientations (e.g., using 1-degree resolution) to reduce the runtime overhead in terms of computing the new velocity. With elliptical pedestrians, the shape of the velocity obstacle, the tangents, and the nearest point operations are governed by the pedestrian's orientation and by the orientation of the obstacles.
We use the predicted velocities to estimate the new position of a pedestrian in the next time-step.

% ##################################################################################################################

\begin{figure}
  \centering
  \includegraphics[width =1.0 \linewidth]{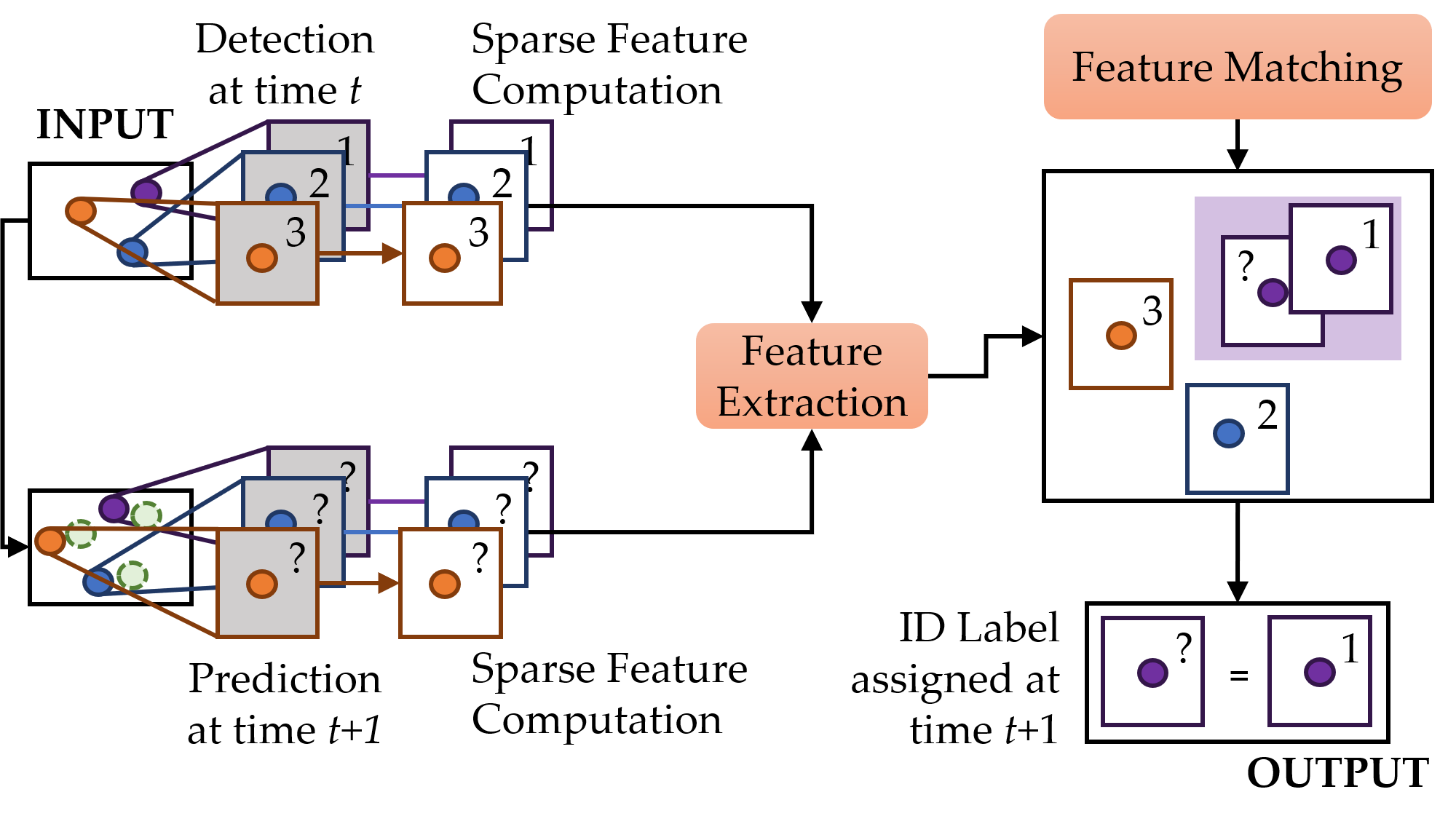}
%   \vspace{-15pt}
  \caption{Overview of our real-time pedestrian tracking algorithm, DensePeds. An input frame at time $t$ undergoes pedestrian detection to generate segmented boxes using Mask R-CNN to compute sparse features. FRVO predicts pedestrian states for frame $t+1$ for which we also compute sparse feature. We extract features using a convolutional neural network, that are then matched using association algorithms. The predicted state receives a new ID depending on the detected person with which it is matched.}
  \label{fig:overview}
  
 \vspace{-20pt}
 \end{figure}

%% file: sim.tex
% \begin{figure*} 
% \centering
%   \subfigure[]{% 
%     \includegraphics[width=.48\textwidth]{nn2.png} \label{nn} 
%   } 
%   \quad 
%   \subfigure[]{% 
%     \includegraphics[width=.48\textwidth]{toy_crop.png} \label{toy} 
%   } 
% %   \vspace{-15pt}
%   \caption{(a) Network architecture of the CNN that is used to generate features. (b) We prove Lemma \ref{lemma1} by showing that the $L_0$ norm of features extracted from a normalized segmented box produced by Mask R-CNN (top) is greater than the $L_0$ norm of features extracted from a regular Faster R-CNN bounding box (bottom) by simulating a simplified version of (a). The output of the non-linear Sigmoid function goes through a final, fully-connected layer that ``flattens'' the binary output matrix into a one-dimensional vector, after which it is easy to see $||\msk||_0 > ||\fst||_0$.} 
% \end{figure*}
% %%%%%%%%%%%%%%%%%%%%%%%%%%%%%%%%%%%%%%%%%%%%%%%SIMULATION%%%%%%%%%%%%%%%%%%%%%%%%%%%%%%%%%%%%%%%%%%%%%%%%%%%

%% file: 4-DensePeds.tex
\section{DensePeds and Sparse Features}\label{Sec4}

In this section, we describe our tracking algorithm, DensePeds, which is based on the tracking-by-detection paradigm. Then, we present our approach for sparse feature extraction and show how they reduce the number of false negatives during tracking.

\subsection{DensePeds: Pedestrian Tracking Algorithm}
Our algorithm uses the detected segmentation boxes given by Mask R-CNN to extract the sparse features. These features are matched using the Cosine metric $d(\cdot, \cdot)$ (ref. the last row of Table~\ref{tab:notations}) and the Hungarian algorithm~\cite{kuhn2010hungarian} and are used to compute the position of each pedestrian during the next frame.  We illustrate our approach in Figure~\ref{fig:overview}. At current time $t$, given the ID labels of all pedestrians in the frame, we want to assign labels for pedestrians at time $t+1$. We start by using Mask R-CNN to implicitly perform pixel-wise pedestrian segmentation to generate segmented boxes, which are rectangular bounding boxes with background subtraction. Next, we predict the spatial coordinates of each person's bounding box for the next time-step using FRVO, as shown in Figure~\ref{fig:frvocomp} and Figure~\ref{fig:frvocomp}. This results in another set of bounding boxes for each pedestrian at time $t+1$.

We use these sets of boxes to compute appearance feature vectors, which are matched using an association algorithm \cite{kuhn2010hungarian}. This matching process is performed in two ways: the Cosine metric and the IoU overlap \cite{iou}. The Cosine metric is used to solve the following optimization problem to identify the detected pedestrian, $h_j$, that is most similar to $\p$. 
\begin{equation}
h^{*}_j = \argmin_{\h}(\mathnormal{d}(\xpi{f}, \xhj{f})| p_i \in \mc{P}, h_j \in \mc{H}_i).
\label{optim}
\end{equation}
The IoU overlap builds a cost matrix, $\zeta$ to measure the amount of overlap of each predicted bounding box with all nearby detection bounding box candidates. $\zeta(i,j)$ stores the IoU overlap of the bounding box of $\Psi_{t+1}$ with that of $\h$ and is calculated as:
\begin{equation}
    \zeta(i,j) = \dfrac{\bb{B}_{\p} \cap \bb{B}_{\h}}{\bb{B}_{\p} \cup \bb{B}_{\h}},\h \in \mc{H}_i.
\end{equation}
Matching a detection to a predicted measurement with maximum overlap thus becomes a max weight matching problem and we solve it efficiently using the Hungarian algorithm \cite{kuhn2010hungarian}. The ID of the pedestrian at time $t$ is then assigned to the best matched pedestrian at time $t+1$.

\subsection{Sparse Features}

In dense crowds, the bounding boxes of nearby pedestrians have significant overlap. This overlap adds noise to the feature vector extracted from the bounding boxes. We address this problem by generating sparse feature vectors by subtracting the noisy background from the original bounding box to produce segmented boxes, which is described next.

\subsubsection{Segmented Boxes Using Mask R-CNN}\label{Sec4.1.1}
We use Mask R-CNN to perform pixel-wise person segmentation. In practice, Mask R-CNN essentially segments out the pedestrian from its bounding box, and thereby reduces the noise that occurs when the pedestrians are nearby.

Mask R-CNN generates a bounding box and its corresponding mask for each detected pedestrian in each frame. We create a white canvas and superimpose a pixel-wise segmented pedestrian onto the canvas using the mask. 
%Effectively, we ``reverse-engineer'' the concept of background subtraction. This is described as follows:
We perform detection at current time $t$ and the output consists of bounding boxes, masks, scores, and class IDs of pedestrians. 

$\mc{B}= \{ \bb{B}_{\h} \ | \ \bb{B} = [\textrm{top left}, m, n], \h \in \mc{H} \}$ denotes the set of bounding boxes, where $\textrm{top left}, m, n$ denote the top left corner, width, height of $\bb{B}_j$, respectively. 

$\mc{M} = \{\bb{M}_{\h} \ | \ \h \in \mc{H}\}$ denotes the set of masks for each $\h$, where each $\bb{M}_{\h}$ is a $[m \times n]$ tensor of booleans. 

Let $\mc{W}= \{\bb{W}_{\h}(\cdot) \ | \ \h \in \mc{H} \}$ be the set of white canvases where each canvas, $\bb{W}_{\h} = [\mathbb{1}]_{m \times n}$, $w$ and $h$. Then, 
\[\mc{U} = \{ \bb{W}_{\h}(\bb{M}_{\h}) \ | \ \bb{W} \in \mc{W}, \bb{M} \in \mc{M}, \h \in \mc{H} \}, \]

\noindent is the set of segmented boxes for each $\h$ at time $t$. These segmented boxes are used by our real-time tracking algorithm shown in Figure~\ref{fig:overview}.

The segmented boxes are input to the DeepSORT CNN \cite{deepsort} to compute binary feature vectors. Since a large portion of a segmented box contains zeros, these feature vectors are mostly sparse. We refer to these vectors as our sparse feature vectors. Our choice of the network parameters governs the size of the feature vectors.
% Sparse feature vectors improve feature matching between pedestrians across frames, thereby, increasing tracking accuracy.

 %Mask R-CNN essentially segments out the pedestrian from its bounding box, thereby removing the noise.
% \input{AuthorKit19/LaTeX/table3.tex}
% \input{AuthorKit19/LaTeX/table4.tex}

%%%%%%%%%%%%%%%%%%%%%%%%%%%%%%%%%%%%%DENSE$$$$$$$$$$$$$$$$$$$$$$$$$$$$$$$
\input{dense.tex}

%%%%%%%%%%%%%%%%%%%%%%%%%%%%%%%%%%%%%DENSE$$$$$$$$$$$$$$$$$$$$$$$$$$$$$$$

\subsubsection{Reduced Probability of Track Loss}
We now show how sparse feature vectors generated from segmented boxes reduce the probability of the loss of pedestrian tracks (false negatives).

% highlight many properties of these features and show how they improve the accuracy of our tracking algorithm.

% After successfully predicting the new pedestrian states, we extract the global appearance features of predicted and detected (and segmented) pedestrians using a deep CNN (Figure \ref{nn}). The architectural and training details of this network are described in \cite{deepsort}. 

% %%%%%%%%%%%%%%%%%%%%%%%%%%%%%%%%%%%%%%%%%%%%%%%SIMULATION%%%%%%%%%%%%%%%%%%%%%%%%%%%%%%%%%%%%%%%%%%%%%%%%%%%
% \input{AuthorKit19/LaTeX/sim.tex}
% %%%%%%%%%%%%%%%%%%%%%%%%%%%%%%%%%%%%%%%%%%%%%%%SIMULATION%%%%%%%%%%%%%%%%%%%%%%%%%%%%%%%%%%%%%%%%%%%%%%%%%%%

We define $\mc{T}_{t} =\{\Psi_{1:t}\}$ to be the set of positively identified states for $\p$ until time $t$. We denote the time since the last update to a track ID as $\mu$. We denote the ID of $\p$ as $\alpha$ and we represent the correct assignment of an ID to $\p$ as $\Gamma(\alpha)$. The threshold for the Cosine metric is $\lambda \underset{\tim{i.i.d.}}{\sim} \bb{U}[0,1]$. The threshold for the track age, \textit{i.e.}, the number of frames before which track is destroyed, is $\xi$. We denote the probability of an event that uses Mask R-CNN as the primary object detection algorithm with $\bb{P}^M(\cdot)$ and the probability of an event that uses a standard Faster R-CNN~\cite{fasterrcnn} as the primary object detection algorithm (\textit{i.e.}, outputs bounding boxes without boundary subtraction) with $\bb{P}^F(\cdot)$. Finally, $\mc{T}_{t} \gets \{\phi\}$ represents the loss of $\mc{T}_{t}$ by occlusion.

We now state and prove the following lemma.
\begin{lemma}
For every pair of feature vectors $(\msk, \fst)$ generated from a segmented box and a bounding box respectively, if $\lVert\msk\lVert_0 > \lVert\fst\rVert_0$, then $\mathnormal{d}(\xpi{f},\msk) < \mathnormal{d}(\xhj{f},\fst) $ with probability $1 - \frac{B}{A}$, where $A$ and $B$ are positive integers and $A > B$.
\label{lemma1}
\end{lemma}
% \vspace{-5pt}
\begin{proof}
Using the definition of the Cosine metric, the lemma reduces to proving the following,

\begin{equation}
\xpi{f}^T(\msk - \fst) > 0
\label{diseq1}
\end{equation}

\noindent We pad both $\msk$ and $\fst$ such that $\lVert\msk\rVert_0 > ||\fst||_0$.
% that three inequalities hold:
% \begin{enumerate}
%     \item The number of 1s in $\xpi{f}^T > \ceil*{l/2}$. 
%     \item The number of 1s in ($\msk - \fst$) $> \ceil*{l/2}$.
%     \item This relationship holds with respect to the number: (number of coinciding locations of 1s in $\xpi{f}^T$ and -1s in ($\msk - \fst$) $<$ (number of coinciding locations of 1s in $\xpi{f}^T$ and 1s in ($\msk - \fst$).
% \end{enumerate}

We reduce $\ft$, $\msk$, and $\fst$ to binary vectors, \textit{i.e.}, vectors composed of $0$s and $1$s. Let $\Delta f = \msk - \fst$. We denote the number of $1$s and $-1$s in $\Delta f$ as $A$ and $B$, respectively. Now, let $x$ and $y$ denote the $L_0$ norm of $\msk$ and $\fst$, respectively. From our padding procedure, we have $x>y$. Then, if $x=A$, and $y=B$, we trivially have $A>B$. But if $y>B$, then $A= x-(y-B) \implies A - B = x - y$. From $x>y$, it again follows that $A>B$. Thus, $x>y \implies A>B$.

Next, we define a $(1,1)$ coordinate in an ordered pair of vectors as the coordinate where both vectors contain $1$s. Similarly, a $(1,-1)$ coordinate in an ordered pair of vectors is the coordinate where the first vector contains $1$ and the second vector contains $-1$. Then, let $p_a$ and $p_b$ respectively denote the number of $(1,1)$ coordinates and $(1,-1)$ coordinates in the pair $(\ft,\Delta f)$. By definition, we have $0<p_a<A$ and $0<p_b<B$. Thus, if we assume $p_a$ and $p_b$ to be uniformly distributed, it directly follows that $\bb{P}(p_a>p_b) = 1-\frac{B}{A}$. 
% \hfill $\blacksquare$

% \noindent Goal: $p_b < p_a$ for all $p_b$
% (condition 3 in Section 4.2.1 for Lemma 1) implies Lemma 1 is true. Now, suppose $A \leq B$. Then $p_b > p_a$ for $0<A<p_b<B$. This gives rise to a contradiction. Therefore, $A>B$. Now let $x$ and $y$ denote the number of 1s in $\msk$ and $\fst$, respectively. Then $x \geq A$ and $y\geq B$. If $y=B$, then $x\geq A > B = y \implies x > y$. If $y>B$, then $A \gets A - (y - B)$. But $A>B \implies A - (y-B)>B \implies A>y$. Since $x\geq A$, $x>y$. Hence, $A>B \implies x>y$.

% \noindent In the backward direction, 

\end{proof}

Based on Lemma~\ref{lemma1}, we finally prove the following proposition.
\begin{prop}
With probability $1-\frac{B}{A}$, sparse feature vectors extracted from segmented boxes decrease the loss of pedestrian tracks, thereby reducing the number of false negatives in comparison to regular bounding boxes.
\label{prop}
\end{prop}
% \vspace{-10pt}
\begin{proof}
In our approach, we use Mask R-CNN for pedestrian detection, which outputs bounding boxes and their corresponding masks. We use the mask and bounding box pair to generate a segmented box (Section~\ref{Sec4.1.1}). The correct assignment of an ID depends on successful feature matching between the predicted measurement feature and the optimal detection feature. In other words, 

\begin{equation}
\mathnormal{d}(\xpi{f},f_{h^{*}_j})>\lambda \Leftrightarrow (\alpha = \phi)
\label{eqn:assign}
\end{equation} 

\noindent Using Lemma~\ref{lemma1} and the fact that $\lambda \underset{i.i.d.}{\sim} \bb{U}[0,1]$, 
\[\bb{P}(\mathnormal{d}(\xpi{f} , f^M_{h^{*}_{j,\p}} ) > \lambda) < \bb{P}(\mathnormal{d}(\xpi{f} ,f^F_{h^{*}_{j,\p}} ) > \lambda)\] 
\noindent Using Eq.~\ref{eqn:assign}, it directly follows that,
\begin{equation}
    \bb{P}^M(\alpha = \phi) < \bb{P}^F(\alpha = \phi)
    \label{alpha}
\end{equation}

\noindent In our approach, we set \[(\mu>\xi) \land  (\alpha = \phi) \Leftrightarrow \mc{T}_{t} \gets \{\phi\}\]
Using Eq.~\ref{alpha}, it follows that,
\begin{equation}
\bb{P}^M(\mc{T}_{t} \gets \{\phi\}) < \bb{P}^F(\mc{T}_{t} \gets \{\phi\})
\label{prob}
\end{equation}
We define the total number of false negatives (FN) as 
\begin{equation}
    FN = \sum_{t=1}^T \sum_{p_g \in \mc{G}} \delta_{\mc{T}_{t}}
    \label{fn}
\end{equation}
where $p_g \in \mc{G}$ denotes a ground truth pedestrian in the set of all ground truth pedestrians at current time $t$ and $\delta_z = 1$ for $z=0$ and $0$ elsewhere. This is a variation of the Kronecker delta function.
Using Eq.~\ref{prob} and Eq.~\ref{fn}, we can say that fewer lost tracks ($\mc{T}_{t} \gets \{\phi\}$) indicate a smaller number of false negatives. 
% \hfill $\blacksquare$

% Finally, by using Lemma \ref{lemma1} with the association formulation Eq.~\ref{optim} described in the next subsection, it is easy to see that $\bb{P}^M(\Gamma(\alpha))>\bb{P}^F(\Gamma(\alpha))$, 

% which completes the proof.
\end{proof}

The upper bound, $\bb{P}^F(\mc{T}_{t})$, in Eq.~\ref{prob} depends on the amount of padding done to $\xpi{f}$ and $\xhj{f}$. A general observed trend is that a higher amount of padding results in a larger upper bound in Eq.~\ref{prob}.

%% file: dense.tex
%%%%%%%%%%%%%%%%%%%%%%%%%%%%%%%%%%%%%%%%%%%%%%%%%%%%% DENSE RESULTS %%%%%%%%%%%%%%%%%%%%%%%%%%%%%%%%%%%%%%%%%%%%%%
\begin{figure*}[t!]
\centering
\includegraphics[width=\textwidth]{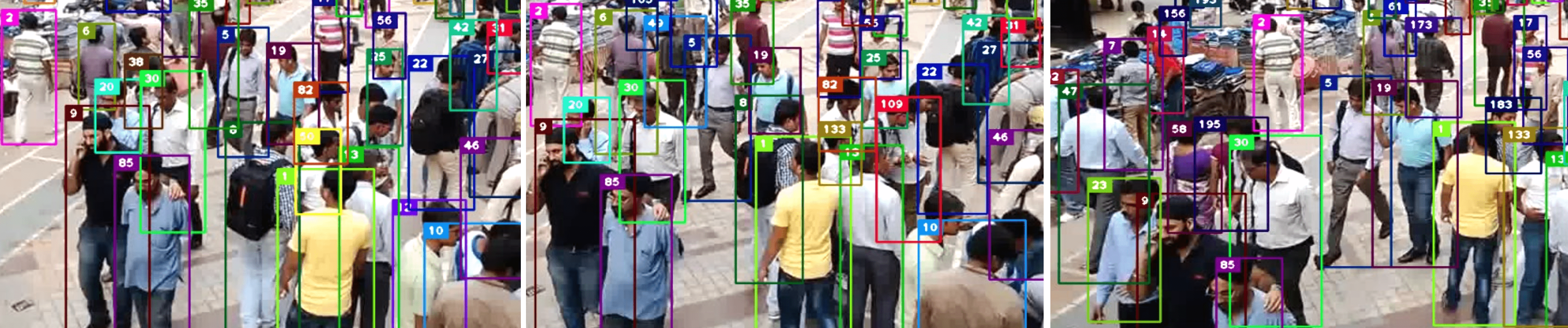}
  
  \vspace{-5pt}
  \caption{Qualitative analysis of DensePeds on the NPLACE-2 sequence consisting of 144 pedestrians in the video. Frames are chosen with a gap of 4 seconds($\sim$ 80 frames). Each bounding box color is associated with a unique ID (displayed on the top left corner of each bounding box). \textbf{Observation:} Note the consistencies in the ID (color), for example, for pedestrians 1 (green), 85 (purple) and 9 (dark red) (at the front, from right to left).}
\end{figure*}
%%%%%%%%%%%%%%%%%%%%%%%%%%%%%%%%%%%%%%%%%%%%%%%%%%%%% DENSE RESULTS %%%%%%%%%%%%%%%%%%%%%%%%%%%%%%%%%%%%%%%%%%%%%%

%% file: 6-Experiments.tex
\section{Experimental Evaluation}\label{Sec5}
\input{compare_on_DPD.tex}
\input{compare_on_MOT.tex}
\input{ablation.tex}
\subsection{Datasets}
% We demonstrate the functionality of our approach on real-world data using the standard MOT15 and MOT16 datasets \cite{mot16}. We select dense sequences that have between $15$-$30$ pedestrians per frame and evaluate our method against state-of-the-art methods on these sequences. The sequences selected correspond to Adl-Rundle-1 from MOT15 and MOT16-08 sequence from MOT16.
% We do not make any complex optimizations to Mask R-CNN results such as multi-scale training and combining different detection features, skip pooling, or random sampling. Consequently, we compare with methods that use public detections for quantitative evaluation.
We evaluate DensePeds on both our dense pedestrian crowd dataset and on the popular MOT~\cite{mot16} benchmark. MOT is a standard benchmark for testing the performance of multiple object tracking algorithms, and it subsumes older benchmarks such as KITTI~\cite{kitti} and PETS~\cite{pets}. When evaluating on MOT, we use the popular MOT16 and MOT15 sequences. However, the main drawback of the sequences in the MOT benchmark is that they do not contain dense crowd videos that we address in this paper.

Our dataset consists of 8 dense crowd videos where the crowd density ranges from 2 to 2.7 pedestrians per square meter. By comparison, the densest crowd video in the MOT benchmark has around 0.2 persons per square meter. All videos in our dataset are shot at an elevated view using a stationary camera. Seven of the eight videos are shot 1080p resolution, and one video is at 480p resolution.

While we agree that the current size of our dataset is prohibitively small for training existing deep learning-based object detectors, we showcase the results of our method on this dataset to validate our claims and analyses. We plan to expand and release a benchmark version of our dataset in the recent future.

\subsection{Evaluation Metrics and Methods}
\subsubsection{Metrics}
We use a subset of the standard CLEAR MOT metrics~\cite{clear} for evaluating the performance of our algorithm. Specifically, the metrics we use are:
\begin{enumerate}
\item \textbf{Number of mostly tracked trajectories (MT).} Correct ID association with a pedestrian across frames for at least 80\% of its life span in a video.
\item \textbf{Number of mostly lost trajectories (ML).} Correct ID association with a pedestrian across frames for at most 20\% of its life span in the video.
\item \textbf{False Negatives (FN).} Total number of false negatives in object detection (\textit{i.e.}, pedestrian tracks lost) over the video.
\item \textbf{Identity Switches (IDSW).} The total number of identity switches over pedestrians in the video.
\item \textbf{MOT Precision (MOTP).} The percentage misalignment between all the predicted bounding boxes and corresponding ground truth boxes over the video.
\item \textbf{MOT Accuracy (MOTA).} Overall tracking accuracy taking into account false positives (FP), false negatives (FN), and ID switches (IDSW). It is defined as $MOTA = 1-\dfrac{FP+FN+IDS}{GT}$, where $GT$ is the sum of annotated pedestrians in all the frames in the video.
\end{enumerate}

These are the most important metrics in CLEAR MOT for evaluating a tracker's performance. We exclude metrics that evaluate the performance of detection since object detection is not a contribution of this paper. We have also excluded the number of false positives from our tables since it is already included in the definition of MOTA and is not the focus of this paper. On the other hand, we highlight the results of false negatives to validate the theoretical formulation of our algorithm in Section~\ref{Sec3.2}.

\subsubsection{Methods} \label{Sec5.2.2}
There is a large number of tracking-by-detection methods on the MOT benchmark, and it is not feasible to evaluate our method against all of them in the limited scope of this paper. Further, many methods submitted on the MOT benchmark server are unpublished. Thus, to keep our evaluations as fair and competitive as possible, we compare with state-of-the-art online trackers on the MOT16 as well as MOT15 benchmarks that are published. By ``state-of-the-art'', we refer to methods that have an average rank greater or equal to our average rank.

For evaluation on our dataset, in addition to online state-of-the-art published methods, we also require their code. The only methods that meet all these criteria are MOTDT~\cite{rtdl3} and MDP~\cite{xiang2015learning}.
% Furthermore, state-of-the-art methods on MOT 15 and MOT16 use DPM detections \cite{dpm}, the detection process for which, is similar to R-CNN. Upon comparing the performance of the two detectors, we observe small differences of 1.2\%, 0.2\%, 0.5\%, and 0.7\% (in tracking accuracy) between the individual sequences, MOT16-02, MOT16-09, MOT16-10, and MOT16-13, respectively. We observe an overall difference of 2\% averaged over the 7 MOT videos. So the performance of Mask R-CNN detector used in our method is nearly identical to the performance of state of the art detectors on MOT. For this reason, we compare our approach against state-of-the-art methods that use public detections, on dense sequences of the MOT benchmark.

% Therefore all our comparisons in Tables 1,2,3 using Mask RCNN are fair comparisons with state of the art methods on MOT benchmarks.

\subsection{Results}
All results generated by DensePeds are obtained without any training on tracking datasets. Thus, DensePeds does not suffer from over-fitting or dataset bias. All the methods that we compare with are trained on the MOT sequences.

\subsubsection{Our Dataset}
Table~\ref{tab:compare_on_DPD} compares DensePeds with MOTDT~\cite{rtdl3} and MDP~\cite{xiang2015learning} on our dense crowd dataset. MOTDT is currently the best online published tracker on the MOT16 benchmark with available code, and MDP is the second-best method with available code. We used their off-the-shelf implementations, with weights pre-trained on the MOT benchmark, to compare with DensePeds.

We observe that DensePeds produces the lowest false negatives of all the methods. Thus, by Proposition~\ref{prop} and Table~\ref{tab:compare_on_DPD}, DensePeds provably reduces the number of false negatives on dense crowd videos.

The low MT and high ML scores for the sequences are not an indicator of poor performance; rather they are a consequence of the strict definitions of the metrics. For example, in dense crowd videos, it is challenging to maintain a track for 80\% of a pedestrian's life span. DensePeds has the highest MT and lowest ML percentages of all the methods. This explains our high number of ID switches as a higher number of pedestrians being tracked would correspondingly increase the likelihood of ID switches. Most importantly, the MOTA of DensePeds is 2.9\% more than that of MOTDT and 2.6\% more than that of MDP. Roughly speaking, this is equivalent to an average rank difference of 19 from MOTDT and 17 from MDP on the MOT benchmark. We do not include the tracking speed metric (Hz) as shown in Table~\ref{tab:compare_on_MOT} since that is a metric produced exclusively by the MOT benchmark server.

\subsubsection{MOT Benchmark}
% We compare with methods on the widely popular and standard tracking benchmarks, MOT16 and MOT17.
Although our algorithm is primarily targeted for dense crowds, in the interest of thorough analysis, we also evaluate DensePeds on sparser crowds. We select a wide range of methods (as described in Section~\ref{Sec5.2.2}) to compare with DensePeds and highlight its relative advantages and disadvantages in Table~\ref{tab:compare_on_MOT}. 

As expected, we do not outperform state of the art on the MOT sequences due to its sparse crowd sequences. We specifically point to our low MOTA scores which we attribute to our algorithm requiring a more accurate ground truth than the ones provided by the MOT benchmark. Our detection and tracking are highly sensitive --- they can track pedestrians that are too distant to be manually labeled. This erroneously leads to a higher count of false positives and reduces the MOTA. This was observed to be true for the methods we compared with as well. Consequently, we exclude FP from the calculation of MOTA for all methods in the interest of fair evaluation.

However, following Proposition~\ref{prop}, our method achieves the lowest number of false negatives among all the methods on the MOT benchmark as well. In terms of runtime, we are approximately 4.5$\times$ faster than the state-of-the-art methods on an NVIDIA Titan Xp GPU.

\subsubsection{Ablation Experiments}
In Table~\ref{tab:ablations}, we validate the benefits of sparse feature vectors (obtained via segmented boxes) in detection and FRVO in tracking in dense crowds with ablation experiments.

First, we replace segmented boxes with regular bounding boxes (boxes without background subtraction), and compare its performance with DensePeds on the MOT benchmark. In accordance with Proposition~\ref{prop}, DensePeds reduces the number of false negatives by 20.7\% on relative.

Next, we replace FRVO in turn with a constant velocity motion model~\cite{deepsort}, the Social Forces model~\cite{helbing1995social}, and the standard RVO model~\cite{van2011reciprocal}; all without changing the rest of the DensePeds algorithm. All of these models fail to track pedestrians in dense crowds (ML nearly 100\%). Note that a consequence of failing to track pedestrians is the unusually low number of ID switches that we observe for these methods. This, in turn, leads to their comparatively high MOTA, despite being mostly unable to track pedestrians. Using FRVO, we decrease ML by 55\% on relative. At the same time, our MOTA improves by a significant 12\% on relative and 8.5\% on absolute.

%% file: compare_on_DPD.tex
\begin{table*}[!htb]
  \centering
%   \scalebox{0.9}{
  \begin{tabular}{|c|c|c|c|c|c|c|c|c|}
  \hline
    Sequence Name & Tracker & MT(\%)$\uparrow$ & ML(\%)$\downarrow$ & IDS$\downarrow$ & FN$\downarrow$ & MOTP(\%)$\uparrow$ & MOTA(\%)$\uparrow$ \\
    \hline
    \multirow{3}{*}{IITF-1} & MOTDT & \textbf{1.8} & 81.1 & 84 (0.2\%) & 10,050 (29.3\%) & \textbf{61.2} & 70.5 \\
    & MDP & \textbf{1.8} & 81.1 & \textbf{40 (0.1\%)} & 10,287 (30.0\%) & \textbf{61.2} & 69.9 \\
    \cline{2-8}
& \textbf{DensePeds} & \textbf{1.8} & \textbf{43.4} & 194 (0.5\%) & \textbf{8,357 (24.3\%)} & 60.7 & \textbf{75.1} \\
    \hline
    \multirow{3}{*}{IITF-2} & MOTDT  & 5.7 & 40.0 & 51 (0.3\%) & 3,902 (23.6\%) & 70.9 & 76.1 \\
& MDP & 14.2 & 25.7 & \textbf{36 (0.2\%)} & 3,033 (18.3\%) & \textbf{72.6} & 81.4 \\
    \cline{2-8}
    & \textbf{DensePeds} & \textbf{17.1} & \textbf{11.5} & 82 (0.5\%) & \textbf{2,308 (14.0\%)} & 70.3 & \textbf{85.5} \\
    \hline
    \multirow{3}{*}{IITF-3} & MOTDT  & 2.0 & 12.5 & 187 (0.5\%) & 7,563 (19.4\%) & 69.5 & 80.1 \\
    & MDP & 8.3 & 16.6 & \textbf{85 (0.2\%)} & 7,735 (19.8\%) & \textbf{71.0} & 79.9 \\
    \cline{2-8}
    & \textbf{DensePeds} & \textbf{22.9} & \textbf{4.2} & 204 (0.5\%) & \textbf{5,968 (15.3\%)} & 69.0 & \textbf{84.2} \\
    \hline
    \multirow{3}{*}{IITF-4} & MOTDT  & 11.4 & 17.2 & 83 (0.5\%) & 3,025 (19.2\%) & 68.2 & 80.2 \\
    & MDP & \textbf{14.2} & 28.5 & \textbf{41 (0.3\%)} & 3,129 (19.9\%) & \textbf{70.3} & 79.8 \\
    \cline{2-8}
    & \textbf{DensePeds} & \textbf{14.2} & \textbf{5.8} & 114 (0.7\%) & \textbf{2,359 (15.0\%)} & 67.6 & \textbf{84.3} \\
    \hline
    \multirow{3}{*}{NDLS-1} & MOTDT  & 0 & 78.3 & 113 (0.3\%) & 9,965 (28.9\%) & 61.8 & 70.8 \\
    & MDP & 0 & 89.1 & \textbf{72 (0.2\%)} & 10,293 (29.8\%) & \textbf{62.0} & 70.0 \\
    \cline{2-8}
    & \textbf{DensePeds} & 0 & \textbf{76.1} & 131 (0.4\%) & \textbf{9,728 (28.2\%)} & 61.5 & \textbf{71.4} \\
    \hline
    \multirow{3}{*}{NDLS-2} & MOTDT  & 0 & 85.2 & 142 (0.3\%) & 14,084 (28.9\%) & 63.0 & 70.8 \\
    & MDP & 0 & 85.1 & \textbf{82 (0.2\%)} & 13,936 (28.6\%) & \textbf{63.9} & 71.2 \\
    \cline{2-8}
    & \textbf{DensePeds} & 0 & \textbf{74.5} & 183 (0.4\%) & \textbf{13,240 (27.2\%)} & 62.7 & \textbf{72.5} \\
    \hline
    \multirow{3}{*}{NPLACE-1} & MOTDT & \textbf{1.9} & 64.7 & 181 (0.5\%) & 9,676 (27.7\%) & \textbf{63.3} & 71.8 \\
    & MDP & \textbf{1.9} & 64.7 & \textbf{153 (0.4\%)} & 9,629 (27.5\%) & 63.1 & 72.0 \\
    \cline{2-8}
    & \textbf{DensePeds} & \textbf{1.9} & \textbf{51.0} & 210 (0.6\%) & \textbf{9,114 (26.1\%)} & 62.8 & \textbf{73.3} \\
    \hline
    \multirow{3}{*}{NPLACE-2} & MOTDT & \textbf{3.8} & 67.3 & 83 (0.3\%) & 7,443 (27.4\%) & \textbf{63.2} & 72.3 \\
    & MDP & 5.7 & 69.2 & \textbf{54 (0.2\%)} & 7,484 (27.6\%) & 62.8 & 72.2 \\
    \cline{2-8}
    & \textbf{DensePeds} & \textbf{3.8} & \textbf{61.6} & 90 (0.3\%) & \textbf{7,161 (26.4\%)} & 62.8 & \textbf{73.3} \\
    \hline
    \multirow{3}{*}{Summary} & MOTDT  & 2.9 & 58.0 & 924 (0.4\%) & 65,708 (26.2\%) & 66.1 & 73.4 \\
    & MDP & 5.1 & 59.9 & \textbf{563 (0.2\%)} & 65,526 (26.1\%) & \textbf{67.4} & 73.7 \\
    \cline{2-8}
    & \textbf{DensePeds} & \textbf{7.0} & \textbf{43.4} & 1,208 (0.5\%) & \textbf{58,235 (23.2\%)} & 65.7 & \textbf{76.3} \\
    \hline
  \end{tabular}
%   }
  \caption{Evaluation on our dense crowds dataset with MOTDT~\cite{rtdl3} and MDP~\cite{xiang2015learning}. MOTDT is currently the best \textit{online} tracker on the MOT benchmark with open-sourced code. Bold is best. Arrows ($\uparrow, \downarrow$) indicate the direction of better performance. \textbf{Observation:} DensePeds improves the accuracy (MOTA) over the state-of-the-art by 2.6\%. We reduce the number of false negatives (FN) by 11\% compared to the next best method, which is a direct consequence of the theoretical formulation in Section~\ref{Sec3.2}.}
%   References: AP\_HWDPL\_p \cite{online151}, AMIR15, AMIR \cite{online152}, RAR\_15\_pub \cite{rtdl5-online153}, HybridDAT \cite{online154}, MOTDT \cite{rtdl1}}
  \label{tab:compare_on_DPD}
\end{table*}
    % & \textbf{DensePeds} & 26/367 & 161/367 & 1,208 & 60,078 & 65.7 & 76.1 \\

%% file: compare_on_MOT.tex
\begin{table}[!htb]
  \centering
  \resizebox{\columnwidth}{!}{%
  \scalebox{1.3}{
  \begin{tabular}{|c|l|c|c|c|c|c|c|c|}
  \hline
    & Tracker & Hz$\uparrow$ & MT(\%)$\uparrow$ & ML(\%)$\downarrow$ & IDS$\downarrow$ & FN$\downarrow$ & MOTP(\%)$\uparrow$ & MOTA(\%)$\uparrow$ \\
    \hline
    \multirow{5}{*}{\rotatebox{90}{MOT15}} & AMIR15 \cite{online152} & 1.9 & 15.8 & \textbf{26.8} & 1026 & 29,397 & 71.7 & 37.6 \\
    & HybridDAT \cite{online154}& 4.6 & 11.4 & 42.2 & 358 & 31,140 & 72.6 & 35.0 \\
%   \cline{2-11}
    & AM \cite{rtdl4} & 0.5 & 11.4 & 43.4 & \textbf{348} & 34,848 & 70.5 & 34.3 \\
%   \cline{2-11}
    & AP\_HWDPL\_p \cite{online151} & 6.7 & 8.7 & 37.4 & 586 & 33,203 & 72.6 & \textbf{38.5} \\
    \cline{2-9}
    & \textbf{DensePeds} & \textbf{28.9} & \textbf{18.6} & 32.7 & 429 & \textbf{27,499} & \textbf{75.6} & 20.0 \\
    \hline
    \multirow{8}{*}{\rotatebox{90}{MOT16}} & EAMTT\_pub \cite{eamtt} & 11.8 & 7.9 & 49.1 & 965 & 102,452 & 75.1 & 38.8 \\
    & RAR16pub \cite{rar} & 0.9 & 13.2 & 41.9 & 648 & 91,173 & 74.8 & 45.9 \\
    & STAM16 \cite{rtdl4} & 0.2 & 14.6 & 43.6 & \textbf{473} & 91,117 & 74.9 & 46.0 \\
    % & JCSTD & 8.8 & 14.4 & 36.4 & 1266 & 86,638 & 74.4 & 47.4 \\
    % & DMMOT \cite{dmmot} & 0.3 & 17.4 & 42.7 & 532 & 89,874 & 73.8 & 46.1 \\
    & MOTDT \cite{rtdl3} & \textbf{20.6} & 15.2 & 38.3 & 792 & 85,431 & 74.8 & \textbf{47.6} \\
    & AMIR \cite{online152} & 1.0 & 14.0 & 41.6 & 774 & 92,856 & \textbf{75.8} & 47.2 \\
    \cline{2-9} 
    & \textbf{DensePeds} & 18.8 & \textbf{20.3} & \textbf{36.1} & 722 & \textbf{78,413} & 75.5 & 40.9 \\
    \hline
  \end{tabular}
  }
  }
\caption{Evaluation on the MOT benchmark with non-anonymous, online methods. In the interest of a fair comparison, we have chosen methods that have an average rank either greater than or equal to our average rank. Bold is best. Arrows ($\uparrow, \downarrow$) indicate the direction of better performance. \textbf{Observation:} DensePeds is up to 4.5$\times$ faster than previous methods and improves accuracy by up to 17\%. Its overall performance has a rank of 23.5 (top 13\%) of all published methods on the MOT15 and 28.8 (top 20\%) on MOT16.}
  \label{tab:compare_on_MOT}
\end{table}

%     & MOTDT17 \cite{online152}  & 18.3 & 250,768 & 17.5 & 35.7 & 50.9 & 76.6 & 2,474  \\
%     %  & TBSS15 \cite{tbss}   & 11.5 & 5,784 & 12.5 & 40.6 & 16.6 & 72.8 & 47  \\
%     % \cline{2-11}
%   \multirow{3}{*}{MOT17} & HAM\_SADF17  & 5.0 & 269,038 & 17.1 & 41.7 & 48.3 & 77.2 & 1,871 \\
% %   \cline{2-11}
%   & DMAN  & 0.3 & 263,608 & 19.3 & 38.3 & 48.2 & 75.7 & 2,194  \\
% %   \cline{2-11}
%     & AM\_ADM17& 5.7 & 265,495 & 13.4 & 39.7 & 48.1 & 76.7 & 2,214 \\
%     &PHD\_GSDL17  & 6.7 & 265,954 & 17.1 & 35.6 & 48.0 & 77.2 & 3,998 \\
%     \cline{2-9}
%     & \textbf{DensePeds}   & \textbf{59.2} & 330,588 & 18.0 & 41.7 & 6.3 & 73.4 & 3,417  \\
%   \hline

%% file: ablation.tex
\begin{table}[!htb]
  \centering
\resizebox{\columnwidth}{!}{%
  \scalebox{1.2}{
    \begin{tabular}{|l|c|c|c|c|c|c|}
    \hline
    Detection & MT(\%)$\uparrow$ & ML(\%)$\downarrow$ & IDS$\downarrow$ & FN$\downarrow$ & MOTP(\%)$\uparrow$ & MOTA(\%)$\uparrow$ \\
    \hline
    BBox & 14.0 & 44.7 & \textbf{313} & 34,716 & \textbf{76.4} & 17.6 \\
    \hline
    \textbf{DensePeds} & \textbf{19.0} & \textbf{32.7} & 429 & \textbf{27,499} & 75.6 & \textbf{20.0} \\
    \hline
    
    \hline
    Motion Model & MT(\%)$\uparrow$ & ML(\%)$\downarrow$ & IDS$\downarrow$ & FN$\downarrow$ & MOTP(\%)$\uparrow$ & MOTA(\%)$\uparrow$ \\
    \hline
    Const. Vel & 0 & 98.0 & 33 & 82,467 & 64.2 & 67.1 \\
    SF & 0 & 96.7 & 145 & 81,936 & 64.5 & 67.3 \\
    RVO & 0 & 96.1 & 155 & 81,832 & 64.4 & 67.8 \\
    \hline
    \textbf{DensePeds} & \textbf{7.0} & \textbf{43.4} & 1,208 & \textbf{58,235} & \textbf{65.7} & \textbf{76.3} \\
    \hline
  \end{tabular}
  }
}
  \caption{Ablation studies where we demonstrate the advantage of Segmented Boxes and FRVO. We replace Segmented Boxes with regular Bounding Boxes (BBox) and compare its performance with DensePeds on the MOT benchmark. We also replace the FRVO model in turn with a constant velocity model (Const. Vel), the Social Forces model (SF)~\cite{helbing1995social} and the standard RVO model (RVO)~\cite{van2011reciprocal} and compare them on our dense crowd dataset. Bold is best. Arrows ($\uparrow, \downarrow$) indicate the direction of better performance. \textbf{Observation:} Using FRVO improves the MOTA by 13.3\% over the next best alternate. Using Segmented Boxes reduces the false negatives by 20.7\%.}
  \label{tab:ablations}
\end{table}